\documentclass[sigconf]{aamas}

\usepackage{balance} 
\usepackage{graphicx}
\usepackage{subfigure}
\usepackage{hyperref}
\usepackage{dsfont}
\usepackage{multicol}
\newcommand{\argmax}{\operatornamewithlimits{argmax}}
\newcommand{\argmin}{\operatornamewithlimits{argmin}}
\newcommand{\E}{\mathbb{E}}
\newcommand{\I}{\mathbb{I}}

\newcommand{\Com}{\mathrm{Com}}
\newcommand{\Cost}{\mathrm{Cost}}
\newcommand{\rad}{\mathrm{rad}}

\usepackage{amsmath}
\usepackage{mathtools}
\usepackage{amsthm}
\usepackage{algorithm}
\usepackage{algorithmic}

\theoremstyle{plain}
\newtheorem{theorem}{Theorem}[section]

\newtheorem{lemma}[theorem]{Lemma}

\theoremstyle{definition}

\theoremstyle{remark}
\newtheorem{remark}[theorem]{Remark}
\usepackage[toc,page]{appendix}

\makeatletter
\gdef\@copyrightpermission{
  \begin{minipage}{0.2\columnwidth}
   \href{https://creativecommons.org/licenses/by/4.0/}{\includegraphics[width=0.90\textwidth]{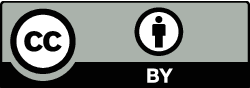}}
  \end{minipage}\hfill
  \begin{minipage}{0.8\columnwidth}
   \href{https://creativecommons.org/licenses/by/4.0/}{This work is licensed under a Creative Commons Attribution International 4.0 License.}
  \end{minipage}
  \vspace{5pt}
}
\makeatother

\setcopyright{ifaamas}
\acmConference[AAMAS '25]{Proc.\@ of the 24th International Conference
on Autonomous Agents and Multiagent Systems (AAMAS 2025)}{May 19 -- 23, 2025}
{Detroit, Michigan, USA}{Y.~Vorobeychik, S.~Das, A.~Nowé  (eds.)}
\copyrightyear{2025}
\acmYear{2025}
\acmDOI{}
\acmPrice{}
\acmISBN{}

\acmSubmissionID{468}

\title[AAMAS-2025 Formatting Instructions]{Learning with Limited Shared Information in Multi-agent Multi-armed Bandit}

\author{Junning Shao}
\affiliation{
  \institution{Tsinghua University}
  \city{Beijing}
  \country{China}}
\affiliation{
  \institution{Shanghai Qi Zhi Institute}
  \city{Shanghai}
  \country{China}}
\email{sjn21@mails.tsinghua.edu.cn}

\author{Siwei Wang}
\affiliation{
  \institution{Microsoft Research}
  \city{Beijing}
  \country{China}}
\email{siweiwang@microsoft.com}

\author{Zhixuan Fang}\thanks{Corresponding authors: Siwei Wang (\texttt{siweiwang@microsoft.com}), Zhixuan Fang (\texttt{zfang@mail.tsinghua.edu.cn}).}
\affiliation{
  \institution{Tsinghua University}
  \city{Beijing}
  \country{China}}
\affiliation{
  \institution{Shanghai Qi Zhi Institute}
  \city{Shanghai}
  \country{China}}
\email{zfang@mail.tsinghua.edu.cn}

\begin{abstract}
Multi-agent multi-armed bandit (MAMAB) is a classic collaborative learning model and has gained much attention in recent years. However, existing studies do not consider the case where an agent may refuse to share all her information with others, e.g., when some of the data contains personal privacy. 
In this paper, we propose a novel limited shared information multi-agent multi-armed bandit (LSI-MAMAB) model in which each agent only shares the information that she is willing to share, and propose the Balanced-ETC algorithm to help multiple agents collaborate efficiently with limited shared information. 
Our analysis shows that Balanced-ETC is asymptotically optimal and its average regret (on each agent) approaches a constant when there are sufficient agents involved. 
Moreover, to encourage agents to participate in this collaborative learning, an incentive mechanism is proposed to make sure each agent can benefit from the collaboration system. Finally, we present experimental results to validate our theoretical results.
\end{abstract}

\keywords{Multi-armed bandit, Multi-agent collaboration, Incentive mechanism}

\newcommand{\BibTeX}{\rm B\kern-.05em{\sc i\kern-.025em b}\kern-.08em\TeX}

\begin{document}


\pagestyle{fancy}
\fancyhead{}


\maketitle 


\section{Introduction}
Multi-armed bandit (MAB) problem is a fundamental theoretical model and has been studied for decades.
There are many practical applications of multi-armed bandit algorithms in industry, especially in on-line platforms \cite{lattimore2020bandit}. 
In a standard MAB game with $N$ arms and time horizon $T$, a single player needs to choose one arm to pull in each time slot, and pulling different arms results in different expected rewards.
Numerous algorithms have been proposed to solve bandit problems, e.g., UCB \cite{auer2002finite, gittins2011multi}, which is a classic algorithm that guarantees $O(N\log T)$ regret upper bound (which matches the $\Omega(N\log T)$ regret lower bound).

However, the standard MAB setting only focuses on the game with a \emph{single} agent, while many real-world applications face the challenge of \emph{multiple} agents making decisions. 
As a concrete example, on an online shopping platform, when a buyer chooses to purchase a certain product, she will receive a corresponding reward. Hence, the buyer can be regarded as an agent and the product can be regarded as an arm. 
Different from the standard MAB setting with a single player, there are always a large number of buyers choosing products to purchase on the online shopping platform, and collaboration between them could help them learn much faster. 
Therefore, in this paper, we consider the multi-agent multi-armed bandit (MAMAB) model (also called multi-player multi-armed bandit in some literature), which features multiple players playing (and collaborating in) the same instance of an MAB problem together. 
Following previous literature, we focus on the common objective of maximizing the total reward of all agents, i.e., social welfare maximization. This is equivalent to minimizing the total regret of agents in the language of bandit problems. 

Several variants of the MAMAB model have been studied in the existing literature, e.g., letting the agents collaborate to speed up the learning procedure with limited communication (e.g., \cite{agarwal2021multi,shahrampour2017multi,sankararaman2019social}), considering decentralized matching market with multi-armed bandit (e.g., \cite{liu2020competing,zhangmatching,kong2023player}). 
%
%
%
However, these studies do not consider the case where an agent may refuse to share her private information with others,
and may even decide to withdraw from learning if she is forced to share some data that she is not willing to share. 
For example, on an \emph{online shopping platform}, the shared information would be the users' comments 
about the products, which may help other users  make their decisions.
%
%
%
However, a user may not comment on every product she purchases in reality, and she can be reluctant to make comments for many reasons, e.g., because the user regards the comments as her privacy, or because commenting on these products does not directly improve her experiences. 
%
A survey presented in \cite{tsai2006s} indicates that customers are likely to refrain from purchasing certain types of items on online shopping platforms due to privacy concerns. 
This survey reveals that different items often elicit varying levels of sensitivity among users, indicating that there are certain products that are widely accepted for online purchases by the majority of individuals, while there are other products that only a small fraction of people are willing to purchase online.
For example, the study shows that for common products such as office supplies, there is little hesitation about buying them online. However, as the items became more personal or related to sensitive topics such as sex, depression, or adult diapers, hesitation increased. When the items were indicative of behavior that could be associated with criminals or terrorists, such as a book on making bombs and bullets, there was significant reservations and reluctance to purchase online. Furthermore, some studies \cite{tsai2006s, bellman2004international} also find that individuals may have varying privacy considerations for the same product. Certain products are more likely to raise privacy concerns among individuals, but the specific privacy concerns of each person are contingent upon factors such as their values, cultural background, personal experiences, and other related factors. These objective survey findings have motivated us to consider a framework in which users selectively share only a portion of the information they are willing to disclose, while choosing to withhold information they perceive as private.
%
Another example is \emph{federated learning}. In a federated learning framework, it is obligatory to preserve data privacy, and a massive amount of data cannot be shared because of the legal concern. For instance, the EU adopted the General Data 
Protection Regulation (GDPR) \cite{european2016regulation}, which states that any institutions or organizations do not have the authority to share users' data unless they have the permission. 
To the best of our knowledge, we are the first to consider 
this particular structural reality. Under our model, agents are more  willing to participate since he can hold his sensitive information private. 


Given the structure of limited information sharing, the first challenge arises as the shared data may be imbalanced, causing serious bias in data \cite{mehrabi2021survey}. For example, there could be a lot of observed rewards on arm $i$, while only limited observed rewards on arm $j$, since many users are willing to share the information of arm $i$, but only few of them are willing to share the information of arm $j$ (image the case where the data related to arm $j$ is more sensitive).
This may result in a serious over-exploration in bandit problems, leading to an inefficient learning.
%

The second challenge for multi-agent collaboration is how to guarantee individual rationality (IR), i,e, agents can always get non-negative utility from the mechanism \cite{shoham2008multiagent}. For example, consider the extreme case that one agent is willing to share all her information, while all the other agents do not share anything. In this case, the agent that shares information helps the others, but cannot directly benefit from sharing. In this case, she may refuse to join the collaboration, leading to a failure of cooperation.
Thus, it is also important to design mechanisms to ensure that all  agents who join this collaboration system can benefit (i.e., earn more reward than learning alone).

In this paper, we introduce a novel multi-agent multi-armed bandit model named limited shared information multi-agent multi-armed bandit (LSI-MAMAB) to characterize the structure of collaborative learning in reality, 
 in which each agent only shares the information that she is willing to share (e.g., only her received rewards on \emph{some} arms) with the other agents, and only decides to participate in this collaboration when she can benefit from it. 
We solve the above two challenges, and design an algorithm \emph{Balanced-ETC} for the case that each agent only shares the information of some specific arms with the other agents.
On the one hand, the total regret of Balanced-ETC is upper bounded by $O(N\log T)$,  
which is asymptotically the same as the best possible regret that one can do in the trivial case where all the information is shared. This reflects the asymptotic optimality of our algorithm.
On the other hand, with the additional incentive design, Balanced-ETC makes sure that the individual regret of each agent is upper bounded by the regret of running a UCB algorithm in the single-agent setting. This means that our algorithm satisfies individual rationality (IR), i.e., all agents have the motivation to join this collaboration system.

We summarize the main contributions of this paper below:
\begin{itemize}
    \item We propose LSI-MAMAB, a novel multi-agent multi-armed bandit model, in which each agent takes part in the collaboration system by only sharing the information that she is willing to share. 
    This lowers the barrier for agents to enter collaborative learning in practices.
    \item We design the Balanced-ETC algorithm to help multiple agents collaborate efficiently under the constraint that they only share partial information, and prove that its overall regret is $O(N\log T + MN^2)$ ($M$ is the number of agents, $N$ is the number of arms and $T$ is time horizon). This means that the average regret (on each agent) 
    almost approaches a constant when there are sufficiently many agents.
    \item We also design an incentive mechanism in the Balanced-ETC algorithm to encourage agents to participate in collaboration. Analysis shows that the incentive mechanism satisfies IR, i.e.,  for any agent, her individual regret of joining this collaboration system is better than learning as a single agent. 
    
    
\end{itemize}

\section{Related Works}

    \textbf{Multi-agent multi-armed bandit.} Numerous frameworks and algorithms have been proposed to solve various multi-agent multi-armed bandit problems. 
    There are many prior works on different multi-agent multi-armed variants.
    For example, 
    \cite{vial2021robust} considers the setting of including honest and malicious agents who recommend best-arm estimates and arbitrary arms, respectively. \cite{liu2020competing,zhangmatching,kong2023player} 
    add the feature that the agents may have collisions with each other when they are pulling the same arm in each time step. \cite{baek2021fair,yang2022distributed} consider the user or group's set of available arms as a subset of the complete arm set, and the regret for each user or group is based on the choices made within their own arm set.
    Compared with our model, these existing works do not consider the case where an agent may refuse to share all her information with others, and even decide to withdraw if she is forced to share some data that she is not willing to share.
    
    Another strand of literature studies multi-agent multi-armed bandit problems with limited communication. For example, \cite{agarwal2021multi} considers a multi-agent multi-armed bandit framework with limited communication rounds and limits bits in each communication rounds; \cite{madhushani2021call} proposes ComEx, a novel and cost-effective communication protocol for cooperative bandits; in the work of \cite{shahrampour2017multi}, players can only exchange information locally to estimate the global reward confined to a network structure; and \cite{sankararaman2019social} proposes a pairwise asynchronous gossip-based protocol that only needs to exchange a limited number of bits to finish communication. 
Compared with these models, we are more interested in which data users are willing to share, rather than the form in which the data is shared. Therefore, we did not consider the adoption of more efficient sharing methods by users, and regard the communication costs as zero. 

     \textbf{Federated Learning.}
     Federated learning (FL), a decentralized machine learning approach, has gained significant attention in recent years. This emerging paradigm enables training of machine learning models on distributed data sources while preserving data privacy. Several related works have explored the foundations of federated learning. \cite{mcmahan2017communication} introduced the concept of federated learning and proposed a practical framework for training models on decentralized data, which enables collaborative model training across multiple devices. There are numerous foundational works dedicated to addressing the various challenges of federated learning and designing federated learning algorithms from multiple aspects, such as privacy preservation, robustness, efficiency, security, scalability, and performance \cite{kamp2019efficient, pillutla2022robust, mcmahan2017communication, jeong2018communication, sattler2019robust}.  Federated learning also expands to encompass a broad range of applications in healthcare \citep{kaissis2020secure}, manufacturing \citep{qu2020blockchained}, agriculture \citep{durrant2022role}, energy \citep{saputra2019energy}, and other fields. 
     
     There are also some prior works on applying FL in bandit problems. 
     For example, in \cite{shi2021federated,shi2021federated1}, an MAB framework of multiple heterogeneous agents and a global principal is proposed. In this framework, agents' local bandit models are not the same (i.e., different agents may have different expected rewards on the same arm) and the goal of the principal is to find the arm with the largest global mean.
    \cite{zhu2021federated} proposes a framework where agents can only communicate their local data with neighbors in a connected graph.  They propose the FedUCB policy, in which the agents preserve differential privacy of their local data. 
    Compared with our model, \cite{shi2021federated,shi2021federated1} require that all agents must follow the global instructions unconditionally and send all their information to the principal. \cite{zhu2021federated} do not consider the case where agents could refuse to share their private information even with differential privacy. In our model, each agent can limit her shared information as her wish. 
    
    \textbf{Incentivized Learning.} 
    Since its proposal by \cite{frazier2014incentivizing, kremer2014implementing}, significant progress has been made in the field of incentivized learning in multi-armed bandit (MAB) problems. Specifically, there are two distinct lines of research in this area. 

    The first line of research \cite{kremer2014implementing, mansour2015bayesian, mansour2016bayesian, immorlica2018incentivizing} assumes that the principal has access to the complete history of actions and rewards, while the agents do not. In this setting, the principal can incentivize them to learn by leveraging proper information to them.
    
    The second line of research considers a publicly available history of actions and rewards, and the incentives are implemented through compensations. This concept was initially introduced by \cite{frazier2014incentivizing} and further generalized by \cite{han2015incentivizing} in Bayesian settings. In the non-Bayesian case,  \cite{wang2018multi} first studied this approach, and it has been recently extended by \cite{wang2021incentivizing}.

    %

    All the aforementioned works assume that every agent is myopic (i.e., they only do exploitation to maximize their short-term rewards), while we assume that every agent is considering her long-term rewards. 
    Hence, the incentive mechanism in this paper can be very different from theirs. 
    %


\section{Problem Formulation}
\begin{figure*}[t]
    \centering 
    \includegraphics[width=0.8\linewidth]{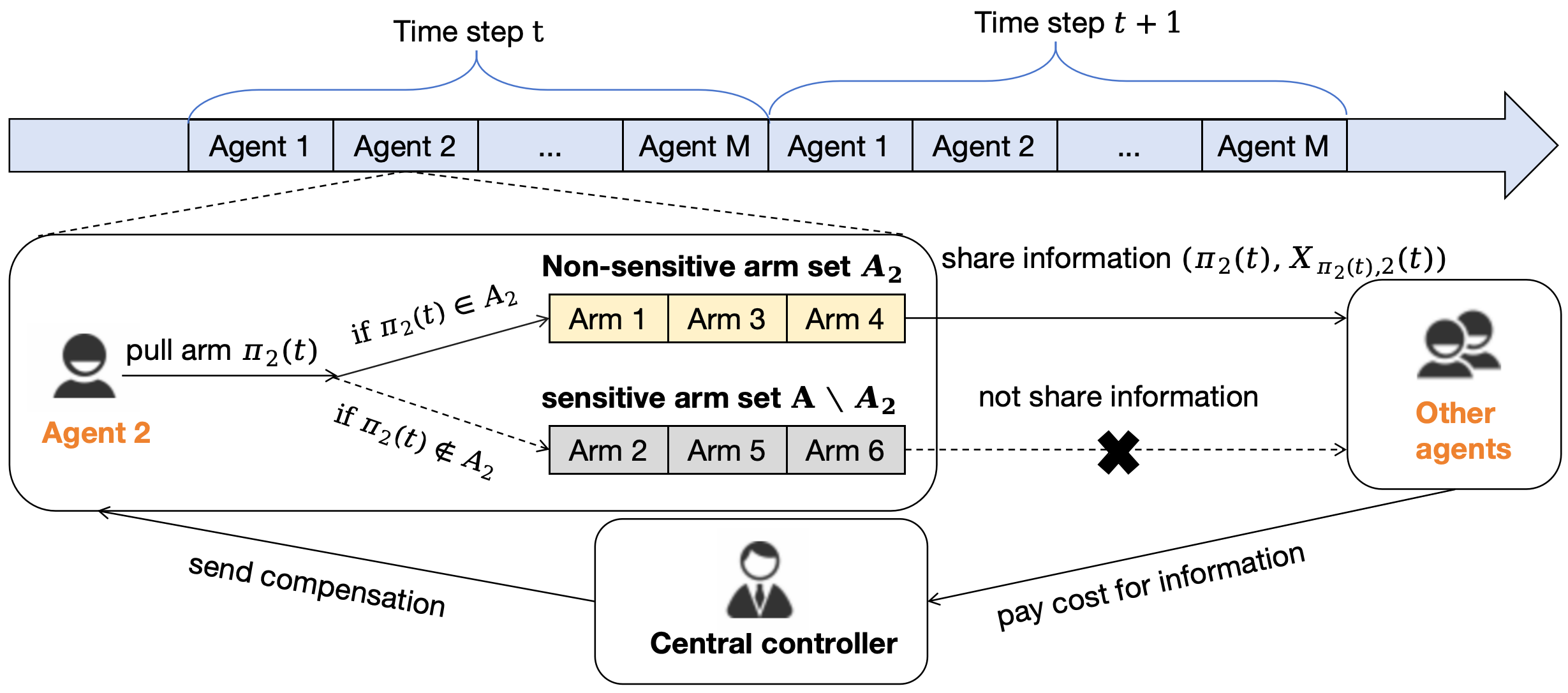}
    \caption{Illustration of the LSI-MAMAB model. In each time step $t$, the agents choose arms sequentially. After agent $m$ pulls an arm $\pi_m(t)$, she then gets a random reward $X_{\pi_m(t), m}(t) \sim D_{\pi_2(t)}$ from the arm $\pi_m(t)$. If $\pi_m(t) \in A_2$, then she could broadcast the reward, so that other agents can use this information. Otherwise, agent $m$ keeps this reward information to herself. 
    The central controller wants to design an algorithm to minimize the overall regret. Besides, he also uses an incentive mechanism to achieve IR: after an agent pulls an arm, he provides some compensation to her, and charges other agents for the shared information (if the arm-reward pair is shared).}
    \Description{Illustration of the LSI-MAMAB model. In each time step $t$, the agents choose arms sequentially. After agent $m$ pulls an arm $\pi_m(t)$, she then gets a random reward $X_{\pi_m(t), m}(t) \sim D_{\pi_2(t)}$ from the arm $\pi_m(t)$. If $\pi_m(t) \in A_2$, then she could broadcast the reward, so that other agents can use this information. Otherwise, agent $m$ keeps this reward information to herself. 
    The central controller wants to design an algorithm to minimize the overall regret. Besides, he also uses an incentive mechanism to achieve IR: after an agent pulls an arm, he provides some compensation to her, and charges other agents for the shared information (if the arm-reward pair is shared).}
    \label{model_figure} 
\end{figure*}

In the the limited shared information multi-agent multi-armed bandit (LSI-MAMAB) model, there exist $M$ independent agents $\{1,2,\cdots, M\}$ and each agent interacts with the same set of arms $A = \{1,2,\cdots,N\}$ for $T$ time steps. 
These arms are homogeneous for different agents, i.e., for any agent $m$, the random reward $X_{i,m}(t)$ of pulling arm $i \in [N]$ is sampled independently from a fixed (but unknown) distribution $D_i$, which is a bounded distribution on $[0,1]$.
We denote $\mu_i$ as the mean of $D_i$, and 
assume that $1\ge \mu_1> \mu_2\ge\cdots \ge \mu_N\ge 0$. We also let $\Delta_i \vcentcolon= \mu_1 - \mu_i$ be the expected reward gap between arm $i$ and the optimal arm and denote $\Delta_{\min} \vcentcolon= \Delta_2$. 
%
Although we assume that there only exists one optimal arm in the model, it is not a necessary condition for obtaining theoretical results in this paper. In fact, all the analysis works in the case that there are multiple optimal arms. 
Moreover, to simplify the analysis of cooperation, we assume that the agents always choose to share truthful information.

To model our assumption that each agent may only share partial information with the others, for any agent $m$, we denote $A_m \subseteq A$ as the set of arms that agent $m$ is willing to share information. We call $A_m$ the non-sensitive arm set and refer to the remaining arm set $A \setminus A_m$ as the sensitive arm set. Note that we allow $A_m$ to be an empty set, i.e., some agents may share nothing with others. 
We also let $S_i = \{m: i \in A_m\}$ be the set of agents who are willing to share the information of arm $i$, and we assume that for any arm $i$, $|S_i| \ge 1$, i.e., there is at least one agent who is willing to share her history information on arm $i$ with others. This assumption is to ensure feasible collaboration, because if the information of some arms is not shared by any agents, no collaboration mechanism can work to reduce the times of exploring these arms.



We illustrate how the LSI-MAMAB model works in Figure \ref{model_figure}. In each time step $t$, the agents choose arms sequentially, i.e., after agent $1$ pulls an arm (and broadcasts the arm-reward pair if she wants to), agent $2$ then chooses an arm to pull. 
%
Note that the order of pulling arms in each time step does not influence any theoretical results in this paper, the assumption that the order being $1,2,\cdots,M$ is only for simplicity. 
After agent $m$ pulls an arm $\pi_m(t)$, she then gets a random reward $X_{\pi_m(t), m}(t) \sim D_{\pi_m(t)}$ from the arm $\pi_m(t)$. If arm $\pi_m(t) \in A_m$ (i.e., results from which agent $m$ feels as non-sensitive information), then she could broadcast the arm-reward pair $(\pi_m(t), X_{\pi_m(t), m}(t))$, so that other agents can use this information. 
Otherwise, agent $m$ keeps this arm-reward pair information private to herself. 
Our goal is to design a collaboration protocol from the perspective of the central controller. On the one hand, we want an algorithm to minimize the overall regret when all the agents follow this protocol (i.e., each agent pulls arms according to the protocol, and share the information as long as the pulled arm is in her non-sensitive arm set), leading to an efficient collaboration. On the other hand, we expect the protocol guarantee a basic incentive of individual rationality, i.e., every agent can benefit from the collaboration compared to individual learning. To achieve this goal, once an agent pulls an arm and share the arm-reward information, we provide some compensation to her, and charge other agents for the shared information.




\subsection{The Overall Regret}

First, for agent $m$, we define her individual regret $R_m(T)$ of a policy as the expected gap between her total reward using the policy and the total reward by pulling the optimal arm for $T$ times, i.e., 
\begin{equation*}
    R_m(T) \triangleq \E\left[\sum_{t=1}^T\mu_{1} - \sum_{t=1}^T \mu_{\pi_m(t)}\right].
\end{equation*}

For different agents, their individual regrets can also be different. 

Next, we define the overall regret as the sum of all agents' individual regrets, i.e.,
\begin{equation*}
    R(T) \triangleq \sum_{m=1}^M R_m(T). 
\end{equation*}

As we can see, $R(T)$ is the overall loss of pulling sub-optimal arms by any agents, which reflects the efficiency of the collaboration. Our goal is to minimize the overall regret $R(T)$.
On the one hand, if everyone shares all her information with others, then it is easy to design algorithms with $R(T) = O(\sum_i {\log T\over \Delta_i})$.
On the other hand, if no one shares her information with others, then the overall regret is at least $\Omega(M\sum_i {\log T\over \Delta_i})$.
Hence, in our setting in which agents only share partial information with others, 
the overall regret should be in between this two bounds, and we want it to be as close to $O(\sum_i {\log T\over \Delta_i})$ as possible.



\subsection{The Individual Rationality} 
From the description above, every agent $m$ will receive some compensation $\Com_m$ at the end of the game because of the information she shares with others, and she also needs to pay some cost $\Cost_m$ because of the information others share with her. 
%
Therefore, the individual regret with incentive mechanism of each agent $m$ consists of three parts: the regret comes from pulling sub-optimal arms ($R_m(T)$); the (minus) compensation she receives ($-\Com_m$); and the cost she needs to pay ($\Cost_m$).

In this paper, we assume that each agent applies the 2-UCB (``2'' represents the factor in confidence radius) policy as her single-player policy, which is one of the most commonly studied asymptotically optimal algorithms. Specifically, in the 2-UCB policy, agents are assumed to pull the arm $\pi(t) = \argmax_{i\in [N]} \{ \hat{\mu}_i(t) + \sqrt{\frac{2 \log T}{N_i(t)}}\}$ at time step $t$,
where $\hat{\mu}_i(t)$ is the empirical mean of arm $i$ and $N_i(t)$ is the number of times that arm $i$ has
been pulled. 
We denote $R_{UCB}(T)$ the expected regret of this 2-UCB policy.

To ensure each agent can benefit from collaboration, we call an incentive mechanism satisfies individual rationality (IR), if for each agent $m$, the following inequality holds: 
\begin{equation*}
 R_m(T) - \Com_m + \Cost_m -R_{UCB}(T) \le 0.
\end{equation*}
As we can see, $R_m(T) - \Com_m + \Cost_m -R_{UCB}(T)$ is the relative regret of joining the collaboration system, and we could say that all agents are willing to participate in the collaboration if the incentive mechanism achieves IR.

The choice of incorporating the 2-UCB algorithm as the baseline algorithm in our definition of IR is motivated by the fact that the UCB algorithm is one of the most classic and well-known bandit algorithms in academic research. The algorithm is theoretically proven to be asymptotically optimal and is often considered as a benchmark algorithm in single-agent scenarios. Hence, we believe that using the UCB algorithm as baseline is highly appropriate.

\section{Minimizing Overall Regret}
\subsection{Balanced-ETC Algorithm}

\begin{algorithm}[t]
\caption{Balanced-ETC (for agent $m$)}\label{algorithm}
\begin{algorithmic}[1]
 \STATE \textbf{Input}: the set of arms $A_m$ that agent $m$ is willing to share, the balance level threshold $B$. 
 \STATE \textbf{Init}: $A(t) = A$
\FOR{$t=1,\dots,T$}
    \STATE $\forall i \in A(t)$, updating their $N_{i}(t)$ and $\hat{\mu}_{i}(t)$.
    \STATE Eliminate arms from $A(t)$ based on Eq. \eqref{Eq_1}.
    \IF {$|A(t)| > 1$ and $|A(t) \cap A_m | > 0$ and $B_m(t) \le B$} 
                \STATE //Explore step 
                \STATE Pull arm $\pi_m(t) = \argmin_{i\in A(t) \cap A_m} N_i(t)$ and receive random reward $X_{\pi_m(t), m}(t)$.
                \STATE Broadcast arm-reward pair $(\pi_m(t), X_{\pi_m(t), m}(t))$.
            \ELSE 
                \STATE//Commit step
                \STATE $\pi_m(t)=\argmax_{i\in A(t)} \hat{\mu}_i(t)$ 
            \ENDIF

\ENDFOR

\end{algorithmic}
\end{algorithm}

In this section, we propose our Balanced-ETC algorithm, which achieves an overall regret upper bound of $O(\sum_i {\log T\over \Delta_i})$ and is asymptotically optimal.

Our Balanced-ETC algorithm follows the elimination framework, i.e., it maintains an active arm set $A(t) \subseteq A$, and eliminates arms from this set if they are sub-optimal with high probability.
%
To guarantee the synchronicity among different agents, the elimination only depends on the publicly shared information, i.e., the arm-reward pairs $(\pi_m(t), X_{\pi_m(t), m}(t))$ that have been broadcast. 
%
Specifically, let $N_i(t)$ be the number of times that arm $i$ has been broadcast, and $\hat{\mu}_i(t)$ be the empirical mean of the expected reward (in the broadcast information). 
Then by concentration analysis, one can see that with high probability, $|\mu_i -\hat{\mu}_i(t) | \le \rad_i(t) \triangleq \sqrt{2\log T/N_i(t)}$. 
%
Therefore, after receiving new information $(\pi_m(t), X_{\pi_m(t), m}(t))$, every agent can update  $N_i(t)$'s and $\hat{\mu}_i(t)$'s for all the arms, and eliminate arm $i$ from the active arm set $A(t)$ if 

\begin{equation}\label{Eq_1}
  \hat{\mu}_i(t) + \rad_i(t) \le \max_{j\in A} \hat{\mu}_j(t) - \rad_j(t).
\end{equation}


%

As mentioned in the introduction, the main challenge here is that the imbalanced distribution of the shared data can result in severe over-exploration. 
For example, if all the agents are willing to share the information of arm $i$, but only one agent is willing to share the information of the optimal arm, the number of broadcast arm-reward pair of sub-optimal arm $i$ will be $M$ times larger than the number of broadcast arm-reward pairs of the optimal arm. In this case,  we need $\Omega({M\log T\over \Delta_i^2})$ explorations to eliminate arm $i$ from the active arm set, which leads to a regret of $\Omega({M\log T\over \Delta_i})$ and is far from optimal.
To tackle this problem, the Balanced-ETC algorithm sets a threshold to restrain over-exploration, i.e., we let the ratio between the maximum $N_i(t)$ and the minimum $N_i(t)$ in the active arm set to be less than the threshold. 
Specifically, when there are still some active arms that agent $m$ is willing to share (i.e., $A(t) \cap A_m \ne \emptyset$), we use $B_m(t) \triangleq {\min_{i\in A(t) \cap A_m} N_i(t)\over \min_{j \in A(t)}N_j(t)}$  to denote the balance level. 
%
As we can see, 
the explorations are more imbalanced when $B_m(t)$ is larger.
Hence, only if the balance level $B_m(t)$ is smaller than some fixed constant $B \ge 1$ (which is given as an input of our algorithm), we let agent $m$ to pull the arm $\argmin_{j\in A(t) \cap A_m} N_j(t)$ and then broadcast the information.
Otherwise, we do not let agent $m$ to explore in this time step, since this can result in severe over-exploration.
The pseudo code of Balanced-ETC is shown in Algorithm \ref{algorithm}.
At the beginning of time step $t$, each agent collects the information shared by other agents and updates $N_i(t)$'s and $\hat{\mu}_i(t)$'s for all the arms. 
Then they use Eq. \eqref{Eq_1} to update the active arm set $A(t)$.
%
%
Only if i) $|A(t)| > 2$ (there still require explorations); ii) $|A(t) \cap A_m | > 0$ (there are still arms that agent $m$ is willing to share and that need explorations); and iii) $B_m(t) \le B$ (exploring this arm will not result in severe over-exploration), agent $m$ will do one explore step, i.e., she chooses to pull arm $\pi_m(t) = \argmin_{i\in A(t) \cap A_m} N_i(t)$ and broadcast arm-reward pair $(\pi_m(t), X_{\pi_m(t), m}(t))$.
Otherwise, she will do one commit step, i.e., she chooses to pull the best active arm $\argmax_{i\in A(t)} \hat{\mu}_i(t)$ and do not share anything with others. 





\subsection{Regret Analysis}

In this section, we provide the overall regret upper bound of our Balanced-ETC algorithm, as well as its proof sketch. Due to space limit, 
detailed proofs
are deferred to the supplementary material. 



\begin{theorem}\label{Th1}
The overall regret of Balanced-ETC can be upper bounded by:
\begin{equation*}
R(T) < \sum_{i = 2}^N  \frac{8(1+\sqrt{B})^2 \log T}{\Delta_{i}} + \frac{4eMN^2}{\Delta_{\min}} + 2MN.
\end{equation*}
\end{theorem}


Note that $B$ is a constant that does not depend on $N,M$ (one can simply choose $B = 1$ in practice). Hence, Theorem \ref{Th1} states that when $T$ is large enough, the overall regret of our algorithm is $O(\sum_{i=2}^N\frac{\log T}{\Delta_i})$. 
This is indeed the best one can do, since the overall regret lower bound is $\Omega(\sum_{i=2}^N\frac{\log T}{\Delta_i})$ even if everyone shares all her information with each other \cite{lai1985asymptotically}.
On the other hand, the average regret of each agent is $O(\sum_{i=2}^N\frac{\log T}{M\Delta_i} + {4eN^2\over \Delta_{\min}})$, which is $M$ times smaller than the regret of the single-agent case, and almost becomes a constant when there are sufficiently many agents in the collaboration.
Hence, our collaboration system is efficient in terms of the scale of players.


\begin{proof}[Proof Sketch of Theorem \ref{Th1}.] We first define a good event
\begin{align*}
 \mathcal{E} = \left\{ \forall t \le T, \forall i \in A, | \hat{\mu}_i(t) - \mu_i| \le \sqrt{3\log T\over 2N_i(t)}\right\},
 \end{align*}
 i.e., for any time step $t$ and any arm $i$, the gap between the empirical mean and the real mean is less than $\sqrt{3\log T\over 2N_i(t)}$. 
 
%

\begin{remark}
    Note that here the confidence radius in event $\mathcal{E}$ is not the same as that we used in Balanced-ETC ($\rad_i(t) = \sqrt{2\log T\over N_i(t)}$). The reason that we choose $\rad_i(t)$ to be larger than $\sqrt{3\log T\over 2N_i(t)}$ is we want to use $N_i(t)$ to obtain both an upper bound and a lower bound for $\Delta_i$. 
    This is crucial for our incentive mechanism (please see details in Section \ref{Section_Incentive}).
\end{remark}

 After applying some concentration inequalities, we have: 
 \begin{lemma}\label{event1}
    The probability of $\mathcal{E}$ happens satisfies
    $\Pr(\mathcal{E}) \ge 1 - \frac{2N}{T}$.
\end{lemma}
Based on Lemma \ref{event1}, we can bound the regret when $\mathcal{E}$ does not happen as 
\begin{equation*}
    \E[R(T) \I[\neg \mathcal{E}]] \le MT \cdot \Pr(\neg\mathcal{E}) \le 2MN.
\end{equation*}
Then we come to bound the regret when $\mathcal{E}$ happens. Conditioned on event $\mathcal{E}$,  for any sub-optimal arm $i \in [2,N]$, the number of explore steps pulling arm $i$ could be upper bounded by the following lemma. 
\begin{lemma}\label{suboptimal}
When event $\mathcal{E}$ happens, the optimal arm $1$ will never be eliminated. Moreover, $\forall i \in [2, N]$,  the number of explore steps pulling arm $i$ could be bounded by:
    $$ N_i(T) \le  \lceil \frac{8 \log T(1+\sqrt{B})^2}{\Delta_i^2} \rceil.$$
\end{lemma}
    The $(1+\sqrt{B})^2$ factor in Lemma \ref{suboptimal} is because that under our Balanced-ETC algorithm, the number of explorations on one active arm can be at most $B$ times larger than the number of explorations on another active arm.
    Therefore, if there are $\Theta\left({\log T \over \Delta_i^2}\right)$ number of explorations on arm $i$, the sum of two confidence radius (sub-optimal arm $i$ and optimal arm $1$) is $\Theta(\Delta_i + \sqrt{B}\Delta_i) = \Theta((1+\sqrt{B})\Delta_i)$, and is larger than $\Theta(\Delta_i)$.
    Only if there are $\Theta\left({(1+\sqrt{B})^2\log T \over \Delta_i^2}\right)$ number of explorations, one could prove that the sum of two confidence radius is $\Theta\left({\Delta_i \over {1+\sqrt{B}}} + {\sqrt{B}\Delta_i \over {1+\sqrt{B}}}\right) = \Theta(\Delta_i)$.


Based on Lemma \ref{suboptimal}, one can easily prove that under event $\mathcal{E}$, the expected regret in the explore steps can be upper bounded by 
    $\sum_{i = 2}^N  \frac{8(1+\sqrt{B})^2 \log T}{\Delta_{i}}$.

Finally, we consider the expected regret of commit steps under event $\mathcal{E}$. In a commit step, the agent pulls the active arm with the highest empirical mean, i.e., $\argmax_{i\in A(t)} \hat{\mu}_i(t)$. To guarantee the accuracy of empirical means $\hat{\mu}_i(t)$'s, we want to first prove that each active arm $i$ is pulled (and shared) for a sufficient number of times. This is stated in the following lemma.
\begin{lemma}\label{lemma_number_explore}
For $\forall t \in [T]$ and $\forall i \in A(t)$, we have 
    $N_i(t) \ge \frac{t}{N} - 1$.
\end{lemma}
Roughly speaking, since there is at least one agent that is willing to share the information of arm $i$ for each arm $i$, in each time step, the active arm $i$ with the least number of $N_i(t)$ must be shared once.
Hence, after $t$ time steps, each active arm must be shared for at least $\frac{t}{N}$ times. 

Based on Lemma \ref{lemma_number_explore}, we know that for any active arm $i$, $N_i(t) = \Theta(t/N)$. Along with the fact that the optimal arm is never eliminated, for any sub-optimal arm $i$, concentration inequalities tell us that its empirical mean $\hat{\mu}_i(t)$ is higher than the empirical mean of the optimal arm $\hat{\mu}_1(t)$ with probability at most $O(\exp(-c\Delta_i^2 t/N^2))$.
Because of this, we have the following lemma.
\begin{lemma}\label{Lemma_Commit}
When event $\mathcal{E}$ happens, the expected regret in commit steps can be bounded by
$\frac{4eMN^2}{\Delta_{\min}}$.
\end{lemma}

Summing over the expected regret in the explore steps and in commit steps when event $\mathcal{E}$ happens, as long as the expected regret when when event $\mathcal{E}$ does not happen, we finally get the regret upper bound in Theorem \ref{Th1}. 
\end{proof}

Although we assume that for any arm $i$, there is at least one agent who is willing to share her history information on arm $i$ with
others, our algorithm can still function properly even if there are arms that have not been shared by any user. 
In this case, each agent would need to independently explore these arms, and the cooperation among users would not accelerate the exploration of these arms.


\subsection{How Much Shared Information Do We Need?}

In this section, we discuss how much shared information (i.e., the number of shared arm-reward pairs) is needed to ensure that the overall regret is close to $O(\sum_i\frac{\log T}{\Delta_i})$.  
In the context of limited information sharing, our primary concern lies in the amount of data that users specifically share. We are more interested in which data users are willing to share, rather than the format in which the shared data is presented. Therefore, we have not taken into consideration the adoption of more efficient sharing methods by users or the compression of data to reduce communication costs.

\begin{theorem}\label{Theorem_Min_Info}
    For any algorithm in the LSI-MAMAB model, if the number of shared arm-reward pairs is $o(\log T)$, then the overall regret is lower bounded by $\Omega(MN \log T)$.
\end{theorem}

The key idea of the proof is that with only $o(\log T)$ shared arm-reward pairs, no one can make sure that which arm is the optimal one, and has to pull all the arms for $\Omega(\log T)$ times by themselves. The detailed proof is deferred to Appendix \ref{appendix E}. 



Theorem \ref{Theorem_Min_Info} tells us that if we want the overall regret to be $O(\sum_i\frac{\log T}{\Delta_i})$, there must be at least $\Omega(\log T)$ number of shared arm-reward pairs.
On the other hand, based on Lemma \ref{suboptimal}, we know that with high probability, the number of shared arm-reward pairs in Balanced-ETC is upper bounded by $O(\sum_i\frac{\log T}{\Delta_i^2})$.
Hence there do not exist algorithms that achieve a similar overall regret upper bound as Balanced-ETC but with a much smaller number of shared arm-reward pairs.
This means that our Balanced-ETC algorithm is very cost-effective. 


\section{Incentive Mechanism in Balanced-ETC}\label{Section_Incentive}


Although the overall regret of Balanced-ETC is asymptotically optimal, the algorithm itself cannot ensure everyone benefits, i.e. , 
some agents may suffer from a higher individual regret (compared with not attending the collaboration system and running a 2-UCB policy herself).
For example, if only one agent is willing to share the information of arm $2$, then in our LSI-MAMAB model, she needs to pull (and share) arm $2$ for more times than running the 2-UCB algorithm alone, since the number of pulls on sub-optimal arms in elimination-based algorithms is always larger than in UCB-based algorithms (see a detailed example in Section \ref{Section_Experiment}). 
Hence, it is necessary to apply some incentive mechanisms to achieve IR. 
%
%



In our incentive mechanism, the center controller is responsible for compensating the agents for sharing their data, and collecting the costs from them for reading the shared data from other agents. The specific amount of cost and compensation are given in Section 5.1. The proposed incentive mechanisms in this paper is objectively present in the real world. For example, in the current existing federal framework, it is common practice for the federated learning platform to provide compensation to the data providers and charge fees to the data users \cite{zeng2021comprehensive}. In our discussion, each agent can simultaneously act as both a data provider and a data user, receiving compensation and incurring charges from the platform. Our algorithms and experiments demonstrate that both agents and the platform can benefit from collaboration, indicating that the platform is viable and agents are willing to participate in the cooperation.

\subsection{Incentive Mechanism}

In our incentive mechanism, at the end of the game, agent $m$ receives $\Com_m$ to compensate her individual regret, where
\begin{equation}\label{Eq_5}
    \Com_m = \sum_{i\in A\setminus A(T)} (N_{i,m}^e(T) + N_{i,m}^c(T))\sqrt{\frac{8(1+\sqrt{B})^2\log T}{N_i(T)}}.
\end{equation}
Here $A(T)$ is the active arm set at time step $T$, $N_{i,m}^e(T)$ is the number of times that agent $m$ shares the information of arm $i$ (i.e., the number of times agent $m$ pulls arm $i$ in an explore step), and $N_{i,m}^c(T)$ is the number of times agent $m$ pulls arm $i$ in a commit step, and $N_i(T)$ is the number of times that arm $i$ has been shared (by all the agents). 
That is, for each time agent $m$ pulls a sub-optimal arm $i \in A\setminus A(T)$, she receives $\sqrt{\frac{8(1+\sqrt{B})^2\log T}{N_i(T)}}$ for compensation.

However, this compensation is much higher than necessary.
In fact, with high probability, this compensation $\Com_m$ is higher than the individual regret $R_m(T)$ of agent $m$.
To make ends meet, we also let all the agents to pay some cost for receiving the shared information, since this information does help them learn and avoid some potential regret.
Specifically, at the end of the game, agent $m$ also needs to pay $\Cost_m$ for the shared information, where
    \begin{equation}\label{Eq_6}
    \Cost_m = \sum_{i \in A\setminus A(T)}N_i(T) \sqrt{\frac{(\sqrt{2}-\sqrt{3/2})^4 \log T}{128(1+\sqrt{B})^2 N_i(T)}}.
\end{equation}

That is, for each time agent $m$ receives a shared arm-reward pair from sub-optimal arm $i \in A\setminus A(T)$, she needs to pay $\sqrt{\frac{(\sqrt{2}-\sqrt{3/2})^4 \log T}{128(1+\sqrt{B})^2 N_i(T)}}$ for this information.

\subsection{Theoretical Analysis}

To understand the meaning of our incentive mechanism, we first introduce the following lemma (detailed proof of which is deferred to supplementary material).

\begin{lemma}\label{range}
When event $\mathcal{E}$ happens and time horizon $T$ is large enough such that $\frac{T - 2N}{\log T} > \frac{8(1+\sqrt{B})^2N}{\Delta_{\min}^2}$ (which means that all the sub-optimal arms must be eliminated), for all sub-optimal arms $i \in [2,N], \Delta_i$ can be bounded by:
\begin{equation*}
    \sqrt{(\sqrt{2}-\sqrt{3/2})^2\log T \over N_i(T)} \le \Delta_i \le \sqrt{\frac{8(1+\sqrt{B})^2\log T}{N_i(T)}}.
\end{equation*}
\end{lemma}

The proof of Lemma \ref{range} is quite straightforward: the second inequality could be obtained by Lemma \ref{suboptimal} directly. 
As for the first inequality, note that under event $\mathcal{E}$, 
\begin{equation*}
    \max_{j\in A(t)} \hat{\mu}_j(t) - \rad_j(t) \le \max_{j\in A(t)} \mu_j \le \mu_1.
\end{equation*}
Hence, we cannot eliminate arm $i$ from the active arm set if $\hat{\mu}_i(t) + \rad_i(t) > \mu_1$.
When $\sqrt{(\sqrt{2}-\sqrt{3/2})^2\log T \over N_i(t)} > \Delta_i$, we know that 
\begin{equation*}
    \hat{\mu}_i(t) + \rad_i(t) \ge \mu_i + (\sqrt{2} - \sqrt{3/2})\sqrt{\log T\over N_i(t)} \ge \mu_i + \Delta_i = \mu_1,
\end{equation*}
which means that arm $i$ cannot be eliminated.
Therefore, since we eliminate arm $i$ from the active arm set at the end of the game, we must have that$\sqrt{(\sqrt{2}-\sqrt{3/2})^2\log T \over N_i(T)} \le \Delta_i$.

\begin{figure*}[ht]
\centering 
\subfigure[Balanced Setting]{ 
\label{Figure_1} 
\includegraphics[width=0.32\linewidth]{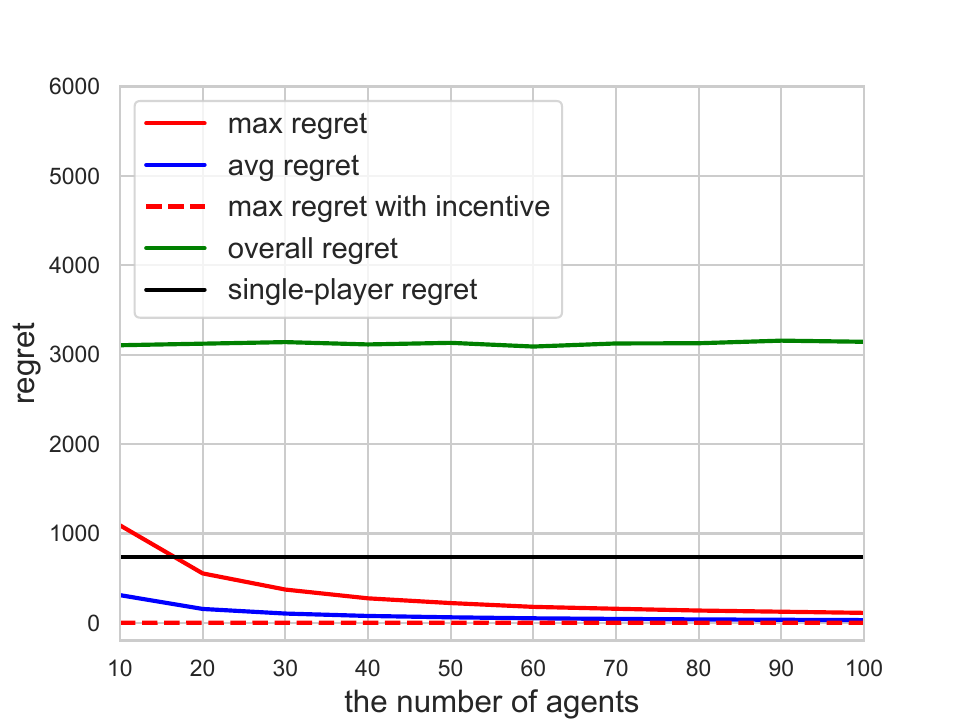}
}
\subfigure[Imbalanced Setting]{ 
\label{Figure_2} 
\includegraphics[width=0.32\linewidth]{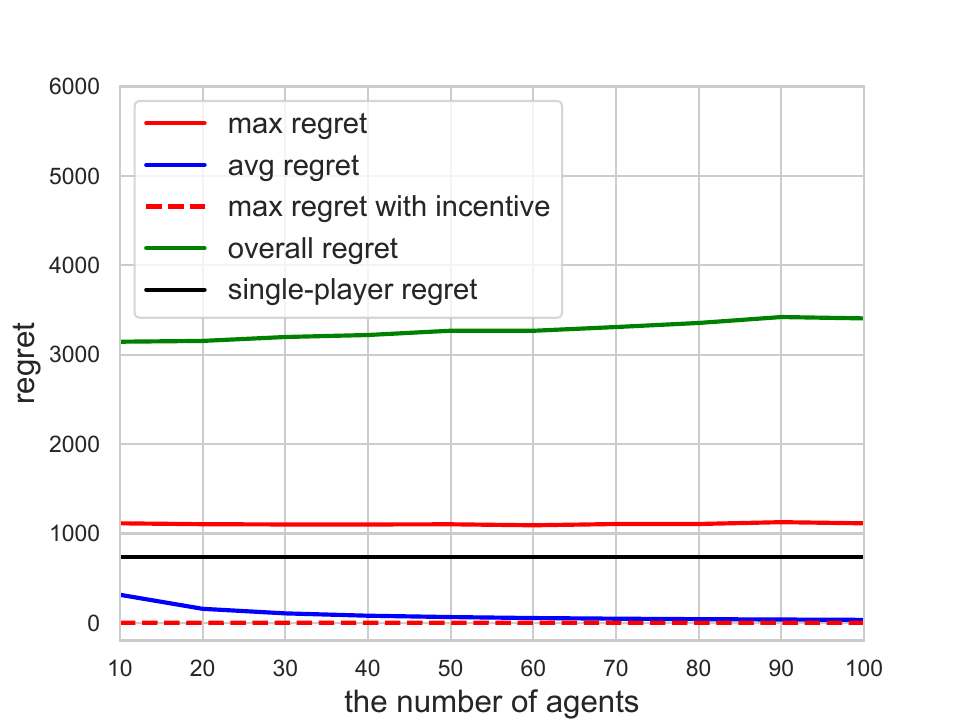}
}
\subfigure[Central Controller's Profits]{ 
\label{Figure_3} 
\includegraphics[width=0.32\linewidth]{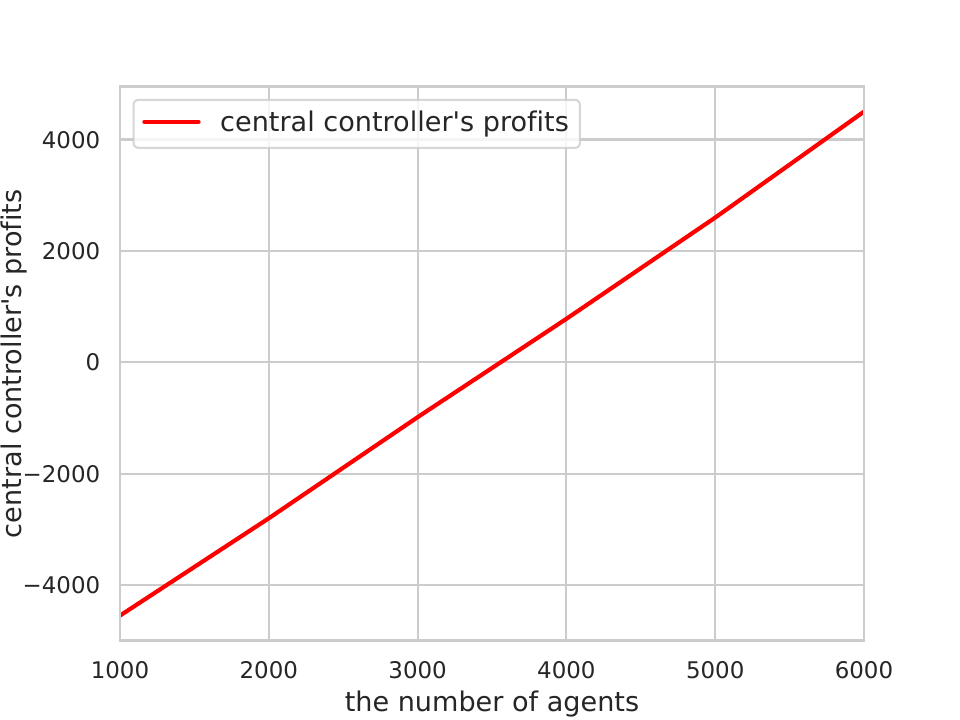}
}
\caption{Experimental results of Balanced-ETC}
\label{fig:image2}
\Description{Experimental results of Balanced-ETC}
\end{figure*}

Lemma \ref{range} tells us a range for the values of $\Delta_i$'s. Hence our incentive mechanism is to let agent $m$ receive compensation as if $\Delta_i$ equals the upper bound (every time she pulls arm $i$) and pay cost as if $\Delta_i$ equals the lower bound (every time an arm-reward pair of arm $i$ is shared to her). 
This makes sense since if we want all the agents to benefit from this collaboration, we must compensate them by the upper bound of their loss, and charge them by the lower bound of their loss.
Besides, when calculating the value of $\Cost_m$, we must consider the ratio between the number of explorations of Balanced-ETC and 2-UCB (in single-agent system).
Since there could be more explorations in Balanced-ETC, the unit cost (of receiving one arm-reward pair of arm $i$) should also be smaller than the lower bound of $\Delta_i$ (as we set in Eq. \eqref{Eq_6}).
%
%

\begin{theorem}\label{Theorem_IR}
    When event $\mathcal{E}$ happens and time horizon $T$ is large enough such that $\frac{T}{\log^2 T} > \frac{N}{4\Delta_{\min}^4}$, applying Eq. \eqref{Eq_5} and Eq. \eqref{Eq_6} as the incentive mechanism in our Balanced-ETC algorithm achieves IR, i.e., for any agent $m$,
    \begin{equation*}
        R_m(T) - \Com_m + \Cost_m - R_{UCB}(T) \le 0. 
    \end{equation*}
\end{theorem}
The proof of Theorem \ref{Theorem_IR} is deferred to Appendix \ref{appendix F}. This result indicates that that every rational agent will join the 
collaborative learning. 
%

Under this incentive mechanism, the compensation that the central controller needs to pay is
\begin{align*}
    \sum_m \Com_m 
    =& \sum_m \sum_{i\in A\setminus A(T)} (N_{i,m}^e(T) + N_{i,m}^c(T))\sqrt{\frac{8(1+\sqrt{B})^2\log T}{N_i(T)}}\\
    =& \sum_{i=2}^N \sqrt{\frac{8(1+\sqrt{B})^2\log T}{N_i(T)}} \left(N_i(T) + \sum_m N_{i,m}^c(T)\right)\\
    =& O\left(\sum_{i=2}^N \sqrt{\log TN_i(T)} + MN^2\sqrt{\frac{\max_{i\ge 2}N_i(T)}{\log T}}\right),
\end{align*}

%
The last equation holds because Lemma \ref{range} tells us $\Delta_i$ is asymptotic to $\sqrt{\frac{\log T}{N_i(T)}}$, which means that the compensation caused by commit steps has the same order as the regret caused by commit steps (Lemma \ref{Lemma_Commit}).
%

On the other hand, the total cost received by the central controller is
\begin{align*}
    \sum_m \Cost_m 
    =& \sum_m \sum_{i \in A\setminus A(T)}N_i(T) \sqrt{\frac{(\sqrt{2}-\sqrt{3/2})^4 \log T}{128(1+\sqrt{B})^2 N_i(T)}}\\
    =& \sum_{i=2}^N M N_i(T) \sqrt{\frac{(\sqrt{2}-\sqrt{3/2})^4 \log T}{128(1+\sqrt{B})^2 N_i(T)}}\\
    =& \Omega\left(M\sum_{i=2}^N \sqrt{\log TN_i(T)}\right).
\end{align*}


%
Although with fewer agents, as seen in Figure \ref{Figure_3}, the central controller might struggle to profit, potentially limiting our mechanism's practicality. However, the overall compensation will be smaller than the overall cost when there are a sufficient number of agents and $T$ is large enough. 
In this case, the central controller can also benefit from making the platform practical in reality.

Note that the conclusion that our incentive mechanism achieves IR is an ex-ante expectation, which is reasonable in the field of algorithm design \cite{hammond1981ex,farina2018ex,babichenko2021bayesian}. 

Specifically, it is an anticipation of the outcomes or results based on available information and analysis prior to the actual occurrence. We focus on this ex-ante expectation since rational agents only participate in cooperation when they anticipate benefiting from it.

\section{Experiments}\label{Section_Experiment}

\balance






In this section, we present experimental results for our Balanced-ETC algorithm and incentive mechanism.
Specifically, in our experiments in main text, there are 10 arms with an expected reward vector
$\mu = [0.9, 0.8, 0.7, 0.6, 0.5, 0.4, 0.3, 0.2, 0.1, 0]$.
As for the information sharing structure, we assume that every agent is only willing to share the information of \emph{one} arm, and consider two settings: the balanced setting and the imbalanced setting. 
In the balanced setting (Figure \ref{Figure_1} and \ref{Figure_3}), $|S_i|$ is the same for all arm $i$, i.e., the number of agents who are willing to share the information of any arm is the same. 
In the imbalanced setting (Figure \ref{Figure_2}), $|S_1| = |S_2| = 1$, while the other arms' $|S_i|$ is the same, i.e., only few agents are willing to share the good arms. 
All these results take an average over 100 independent runs.

In Figure \ref{Figure_1} and Figure \ref{Figure_2}, we set $T = 10^6$, $B = 1$, and compare the performance of Balanced-ETC under different number of agents. 
We can observe that in both the balanced and imbalanced settings, the overall regret (green line) doesn't increase significantly as the number of agents grows. Hence, the \emph{average} individual regret (blue line) keeps decreasing and tends to be zero when there are more agents involved. This accords with our analysis, and demonstrates the effectiveness of our collaboration system.
However, as we can see, if we do not apply any incentive mechanism, then the max original individual regret (red line) can be larger than the regret of running the 2-UCB policy alone (black line), especially when there are few agents involved or when the sharing structure is imbalanced. 
After adding our incentive mechanism (with compensation $\Com_m$ and cost $\Cost_m$), the maximum incentive individual regret (red imaginary line) is always lower than the regret of running the 2-UCB policy alone (black line), and becomes almost $0$. 
%
This also accords with our analysis, and demonstrates that our incentive mechanism can achieve IR, i.e., 
it makes sure that every agent who joins this collaboration system can benefit. 
In Figure \ref{Figure_3}, we set $T = 10^6$, $B = 1$, and compare the profits of the central controller (i.e., $\sum_m \Cost_m - \sum_m \Com_m$) under different number of agents. 
We can see that the profit increases linearly in the number of agents, and 
 surpasses 0 at around 4k agents.
%
This also accords with our analysis, and empirically shows that the central controller can also benefit from making the platform practical in reality, as long as there are sufficient agents participating.
In addition to the balanced and imbalanced settings, we also designed experiments under a random setting, with randomized reward distributions and information-sharing structures, to test the robustness of our algorithm. The results demonstrate that our algorithm continues to perform effectively in random setting. Due to space limitations, these experiments are detailed in Appendix \ref{appendix I}.

\section{Conclusion}

In this paper, we propose the LSI-MAMAB model, and design the Balanced-ETC algorithm and a corresponding incentive mechanism. This is the first work on MAMAB
problem with limited information sharing, which sheds light on many collaborative learning scenarios with data sharing constraints.
We show that our algorithm's overall regret is asymptotically optimal and our incentive mechanism achieves individual rationality both theoretically and empirically.

We believe that further research can build on the current work to inspire deeper exploration of limited shared information structures. For example, considering the trade-offs between the costs of information exchange and the effectiveness of cooperation, or modeling heterogeneous user preferences. There remain significant challenges in developing learning frameworks that provide users with greater autonomy and better privacy protection.

\begin{acks}
The work of Junning Shao and Zhixuan Fang is supported by Tsinghua University Dushi Program and Shanghai Qi Zhi Institute Innovation Program SQZ202312.
The work of Siwei Wang is supported in part by the National Natural Science Foundation of China Grant 62106122.
\end{acks}


\bibliographystyle{ACM-Reference-Format} 
\bibliography{sample}

\newpage
\onecolumn
\begin{appendices}
\section{Proof of Lemma \ref{event1}} \label{appendix A}

We first define the good event as
\begin{align*}
 \mathcal{E} = \left\{ \forall t \le T, \forall i \in A, | \hat{\mu}_i(t) - \mu_i| \le \sqrt{3\log T\over 2N_i(t)}\right\},
 \end{align*}
 i.e., for any time step $t$ and any arm $i$, the gap between its empirical mean and its real mean is less than $\sqrt{3\log T\over 2N_i(t)}$. 
 %
 %
%
 
\noindent\textbf{Lemma \ref{event1}.}
    The probability of $\mathcal{E}$ happens satisfies
\begin{equation*}
    \Pr(\mathcal{E}) \ge 1 - \frac{2N}{T}.
\end{equation*}

\begin{proof}
Using the Hoeffding inequality \cite{hoeffding1994probability}, we obtain
\begin{align*}
    \Pr( \neg \mathcal{E}) &= \Pr \left [ \exists t \le T,\exists i \in A, | \hat{\mu}_i(t) - \mu_i| > \sqrt{3\log T\over 2N_i(t)} \right]\\
    &\le \sum_{i\in A}\sum_{t=1}^T \Pr\left[ | \hat{\mu}_i(t) - \mu_i| > \sqrt{3\log T\over 2N_i(t)} \right]\\
    &\le \sum_{i\in A}\sum_{t=1}^T \sum_{\tau = 1}^{t-1} \Pr\left[ | \hat{\mu}_i(t) - \mu_i| > \sqrt{\frac{3\log T}{2\tau}}, N_i(t) = \tau \right]\\
    &\le \sum_{i\in A}\sum_{t=1}^T \sum_{\tau = 1}^{t-1} \frac{2}{T^3} \\
    &\le \frac{2N}{T}.
\end{align*}
    
\end{proof}



\section{Proof of Lemma \ref{suboptimal}}\label{appendix B}


\textbf{Lemma \ref{suboptimal}.} 
When event $\mathcal{E}$ happens, the optimal arm $1$ will never be eliminated. Moreover, $\forall i \in [2, N]$,  the number of explore steps pulling arm $i$ could be bounded by:
    $$ N_i(T) \le  \lceil \frac{8 \log T(1+\sqrt{B})^2}{\Delta_i^2} \rceil.$$

\begin{proof}
When event $\mathcal{E}$ happens, for any sub-optimal arm $i$, it holds that 
\begin{eqnarray*}
\hat{\mu}_1(t) + \rad_1(t) &\ge & \hat{\mu}_1(t) + \sqrt{3\log T\over 2N_i(t)}\\
&\ge& \mu_1\\
&\ge& \mu_i\\
&\ge& \hat{\mu}_i(t) - \sqrt{3\log T\over 2N_i(t)}\\
&\ge& \hat{\mu}_i(t) - \rad_i(t).
\end{eqnarray*}
Thus, the optimal arm $1$ will never be eliminated.

Then for sub-optimal arm $i$, we prove the upper bound of $N_i(T)$ by contradiction.  Further, when 
    $$N_i(T) > \lceil \frac{8 \log T(1+\sqrt{B})^2}{\Delta_i^2} \rceil,$$
    it means $\exists t\in[T], \exists m \in [M], N_i(t) = \lceil \frac{8 \log T(1+\sqrt{B})^2}{\Delta_i^2} \rceil, \pi_m(t) = i$. According to our algorithm, we have $N_i(t)/N_1(t)\le B_m(t) \le B$. 
    Hence, under event $\mathcal{E}$, it holds that
    \begin{align*}
        \hat{\mu}_i(t) + \rad_i(t) & \le \mu_i + 2\rad_i(t) \\
        & \le \mu_1 -\Delta_i + 2\rad_i(t) \\
        & \le \hat{\mu}_1(t) +\rad_1(t) +2\rad_i(t) - \Delta_i\\
        & \le \hat{\mu}_1 - \rad_1(t) + 2\rad_1(t) +2\rad_i(t) - \Delta_i.
    \end{align*}
    However, when $N_i(t) \ge \lceil \frac{8 \log T(1+\sqrt{B})^2}{\Delta_i^2} \rceil$, we know that $N_1(t) \ge N_i(t) / B$, which implies $2\rad_1(t) +2\rad_i(t) - \Delta_i \le 0$. 
    Therefore, it holds that
    $$\hat{\mu}_i(t) + \rad_i(t) \le \hat{\mu}_1(t) - \rad_1(t),$$
    which means arm $i$ is eliminated by Eq. (\ref{Eq_1}).
    This leads to a contradiction and thus proves the lemma.
\end{proof}

\section{Proof of Lemma \ref{lemma_number_explore}} \label{appendix C}

\textbf{Lemma \ref{lemma_number_explore}.}
For $\forall t \in [T]$ and $\forall i \in A(t)$, we have 
$$N_i(t) \ge \frac{t}{N} - 1.$$

\begin{proof}
We use induction to prove this lemma.

Base case: For $t = 1, ..., N$, the statement holds.

Induction step: Assume the statement holds for $\forall t < t'$. It means $\min_{j\in [N]}{N_j(t)} \ge \frac{t}{N} - 1$.

In our algorithm, at least one active arm with the smallest number $N_j(t)$ will be pulled in each time step. It means that for $\forall t \in [t', t'+N]$, we have
\begin{align*}
\min_{j\in [N]}{N_j(t)} &\ge \min_{j\in [N]}{N_j(t-N)} + 1 \\
&\ge \frac{t-N}{N} - 1 + 1 \ge \frac{t}{N} - 1   
\end{align*}

That is, the statement also holds for $\forall t \in [t', t'+N]$, establishing the induction step.
%
\end{proof}

\section{Proof of Lemma \ref{Lemma_Commit}}\label{appendix D}

\textbf{Lemma \ref{Lemma_Commit}.} When event $\mathcal{E}$ happens, the expected regret in commit steps can be bounded by
$\frac{4eMN^2}{\Delta_{\min}} $.

\begin{proof}
Let $R^c(T)$ be the expected regret in commit steps, we have
\begin{align*}
    \E[R^c(T)] & \le \E\left[\sum_{m=1}^M\sum_{t=1}^T \sum_{i=2}^N \Delta_i \I\{\pi_{m}(t) = i\}\right]\\
            & \le \E\left[\sum_{m=1}^M\sum_{t=1}^T \sum_{i=2}^N \Delta_i \I\{\hat{\mu}_i(t) > \hat{\mu}_1(t)\}\right] \\
            & \le \E\left[\sum_{i=m}^M\sum_{t=1}^T \sum_{i=2}^N \Delta_i \left(\I\{\hat{\mu}_1(t) - \mu_1 \le -\Delta_i/2\} + \I\{\hat{\mu}_i(t) - \mu_i > \Delta_i/2\}\right) \right].
\end{align*}
Using the Hoeffding inequality \cite{hoeffding1994probability}, we have
$$
\Pr(\hat{\mu}_1(t) - \mu_1 \le -\Delta_i/2) \le e^{-N_i(t)\Delta_i^2/2}  \le e^{-(\frac{t}{N} - 1) \Delta_i^2/2};
$$
$$
\Pr(\hat{\mu}_i(t) - \mu_i > -\Delta_i/2) \le e^{-N_i(t)\Delta_i^2/2} \le e^{-(\frac{t}{N} - 1) \Delta_i^2/2}.
$$
The we can obtain that
\begin{align*}
        \E[R^c(T)] & \le \E\left[\sum_{m=1}^M\sum_{t=1}^T \sum_{i=2}^N \Delta_i (\I\{\hat{\mu}_1(t) - \mu_1 \le -\Delta_i/2\} + \I\{\hat{\mu}_i(t) - \mu_i > \Delta_i/2\})\right]\\
        & \le \E\left[\sum_{m=1}^M \sum_{i=2}^N \sum_{t=1}^T  \Delta_i 2e^{-(\frac{t}{N} - 1) \Delta_i^2/2} \right]\\
        & \le \E\left[\sum_{m=1}^M\sum_{i=2}^N \Delta_i \frac{4eN}{\Delta_i^2} \right] \le \frac{4eMN^2}{\Delta_{\min}}.
\end{align*}
\end{proof}

\section{Proof of Theorem \ref{Th1}} \label{appendix Th1}
\textbf{Theorem \ref{Th1}.}
The overall regret of Balanced-ETC can be upper bounded by:
\begin{equation*}
R(T) < \sum_{i = 2}^N  \frac{8(1+\sqrt{B})^2 \log T}{\Delta_{i}} + \frac{4eMN^2}{\Delta_{\min}} + 2MN.
\end{equation*}
\begin{proof} We first define a good event
\begin{align*}
 \mathcal{E} = \left\{ \forall t \le T, \forall i \in A, | \hat{\mu}_i(t) - \mu_i| \le \sqrt{3\log T\over 2N_i(t)}\right\},
 \end{align*}

 i.e., for any time step $t$ and any arm $i$, the gap between the empirical mean and the real mean is less than $\sqrt{3\log T\over 2N_i(t)}$.

Based on Lemma \ref{event1}, we can bound the regret when $\mathcal{E}$ does not happen as 
\begin{equation*}
    \E[R(T) \I[\neg \mathcal{E}]] \le MT \cdot \Pr(\neg\mathcal{E}) \le 2MN.
\end{equation*}
Then we come to bound the regret when $\mathcal{E}$ happens.

Based on Lemma \ref{suboptimal}, we prove that under event $\mathcal{E}$, the expected regret in the explore steps can be upper bounded by 
    $\sum_{i = 2}^N  \frac{8(1+\sqrt{B})^2 \log T}{\Delta_{i}}$.

Based on Lemma \ref{Lemma_Commit}, we can upper bound the expected regret of commit steps under event $\mathcal{E}$ by $\frac{4eMN^2}{\Delta_{\min}}$. 

Summing over the expected regret in the explore steps and in commit steps when event $\mathcal{E}$ happens, as long as the expected regret when when event $\mathcal{E}$ does not happen, we finally get the regret upper bound in Theorem \ref{Th1}. 
\end{proof}

\section{Proof of Theorem \ref{Theorem_Min_Info}}\label{appendix E}
\textbf{Theorem \ref{Theorem_Min_Info}.} For any algorithm in the LSI-MAMAB model, if the number of shared arm-reward pairs is $o(\log T)$, then the overall regret is lower bounded by $\Omega(MN \log T)$.
\begin{proof}
    Existing analysis \cite{lai1985asymptotically} proves the following asymptotic lower bound in classic bandit case: 
\begin{equation}
    \lim_{T\to \infty}\inf\frac{\E[N_i(T)]}{\log(T)} \ge \frac{1}{KL(D_i, D_1)},
\end{equation}
where $N_i(T)$ is the number of times that we pull arm $i$ until time step $T$, and $KL(D_i, D_1)$ is the KL-divergence between two reward distributions. 

Thus, in our LSI-MAMAB model, each agent must observe (either by pulling the arm herself or by receiving the shared information from the others) the information of each arm for at least $\Omega(\log T)$ times. 
If the number of shared arm-reward pairs is $o(\log T)$, then each agent $m$ needs to pull every sub-optimal arm $i$ for at least $\Omega(\log T)$ times.
Hence her individual regret is at least $\Omega(N \log T)$, and the overall regret is at least $\Omega(MN \log T)$.

\end{proof}

\section{Proof of Lemma \ref{range}} \label{appendix F}
\textbf{Lemma \ref{range}.} When event $\mathcal{E}$ happens and time horizon $T$ is large enough such that $\frac{T - 2N}{\log T} > \frac{8(1+\sqrt{B})^2N}{\Delta_{\min}^2}$ (which means that all the sub-optimal arms must be eliminated), for all sub-optimal arms $i \in [2,N], \Delta_i$ can be bounded by:
\begin{equation*}
    \sqrt{(\sqrt{2}-\sqrt{3/2})^2\log T \over N_i(T)} \le \Delta_i \le \sqrt{\frac{8(1+\sqrt{B})^2\log T}{N_i(T)}}.
\end{equation*}
\begin{proof}
 The proof of Lemma \ref{range} is quite straightforward: the second inequality could be obtained by Lemma \ref{suboptimal} directly. 
As for the first inequality, note that under event $\mathcal{E}$, 
\begin{equation*}
    \max_{j\in A(t)} \hat{\mu}_j(t) - \rad_j(t) \le \max_{j\in A(t)} \mu_j \le \mu_1.
\end{equation*}
Hence, we cannot eliminate arm $i$ from the active arm set if $\hat{\mu}_i(t) + \rad_i(t) > \mu_1$.
When $\sqrt{(\sqrt{2}-\sqrt{3/2})^2\log T \over N_i(t)} > \Delta_i$, we know that 

\begin{equation*}
    \hat{\mu}_i(t) + \rad_i(t) \ge \mu_i + (\sqrt{2} - \sqrt{3/2})\sqrt{\log T\over N_i(t)} \ge \mu_i + \Delta_i = \mu_1,
\end{equation*}
which means that arm $i$ cannot be eliminated.
Therefore, since we eliminate arm $i$ from the active arm set at the end of the game, we must have that$\sqrt{(\sqrt{2}-\sqrt{3/2})^2\log T \over N_i(T)} \le \Delta_i$.   
\end{proof}

\section{Proof of Lemma \ref{range_ucb}}\label{appendix G}

\begin{lemma}\label{range_ucb}
Consider we are running the 2-UCB policy in a single-agent system.
Let $N_i'(t)$ be the number of pulls on arm $i$ until time step $t$, $\hat{\mu}_i'(t)$ be the empirical mean and $\mathcal{E}'$ be the event that
\begin{align*}
 \mathcal{E}' = \left\{ \forall t \le T, \forall i \in A, | \hat{\mu}_i'(t) - \mu_i| \le \sqrt{3\log T\over 2N_i'(t)}\right\}.
 \end{align*}
Then if event $\mathcal{E}'$ happens and time horizon $T$ is large enough such that $\frac{T}{\log^2 T} > \frac{N}{4\Delta_{\min}^4}$, for all sub-optimal arms $i \in [2,N]$, we have that 
\begin{equation*}
    \Delta_i \ge \sqrt{(\sqrt{2}-\sqrt{3/2})^2\log T \over 4N_i'(T)}. 
\end{equation*}
\end{lemma}

\begin{proof}
    We prove the lemma by contradiction. For any arm $i$, denote $L_i = {(\sqrt{2}-\sqrt{3/2})^2\log T \over 4\Delta_i^2}$. 

    If event $\mathcal{E}'$ happens, but 
    $\exists k \in [N], N_k'(T) < L_k$. 
    Then, we divide $[0, T]$ into $L_k$ blocks with length $\frac{T}{L_k}$. By the pigeonhole principle, there must exist one block $[t_1, t_2]$, in which arm $k$ is not pulled, i.e. , $N_k'(t_1) = N_k'(t_2) < L_k$.

When event $\mathcal{E}'$ happens, we have $\forall t \in [t_1, t_2]$
\begin{align*}
    \hat{\mu}_k(t) + \rad_k(t) &= \hat{\mu}_k(t) + \sqrt{\frac{3\log T}{2N_k'(t)}} + (\sqrt{2} - \sqrt{3/2})\sqrt{\frac{\log T}{N_k'(t)}}\\
    &\ge \mu_k + (\sqrt{2} - \sqrt{3/2})\sqrt{\frac{\log T}{N_k'(t)}}\\
    &\ge \mu_k + (\sqrt{2} - \sqrt{3/2})\sqrt{\frac{\log T}{L_k}}\\
    &\ge \mu_k + 2\Delta_k \\
    &\ge \mu_1 + \Delta_k.
\end{align*}

When $\frac{T}{\log^2 T} > \frac{N}{4\Delta_{\min}^4}$, it holds that 
$$\exists i \in [N], N_i'(t_2) - N_i'(t_1) > \frac{4(\sqrt{3/2}+\sqrt{2})^2\log T}{\Delta_k^2}.$$
Otherwise, we have 
\begin{align*}
    t_2 - t_1 &= \sum_{j=1}^N N_j'(t_2) - N_j'(t_1)\\
    &\le \sum_{j=1}^N \frac{4(\sqrt{3/2}+\sqrt{2})^2\log T}{\Delta_k^2} \\
    & \le   \frac{4N(\sqrt{3/2}+\sqrt{2})^2\log T}{\Delta_{\min}^2}\\
    & < \frac{4T\Delta_{\min}^2}{(\sqrt{2}-\sqrt{3/2})^2 \log T}\\
    & < \frac{4T\Delta_{k}^2}{(\sqrt{2}-\sqrt{3/2})^2 \log T}\\
    & < \frac{T}{L_k},
\end{align*}
which contradicts with $t_2 - t_1 = \frac{T}{L_k}$.

Then for this arm $i$, it holds that $\exists t_3 \in [t_1, t_2], N_i'(t_3) - N_i'(t_1) = \frac{4(\sqrt{3/2}+\sqrt{2})^2\log T}{\Delta_k^2}$ and arm $i$ is pulled at time $t_3$. For arm $i$ and time $t_3$, we have
\begin{align*}
    \hat{\mu}_i(t_3) + \rad_i(t_3) &\le \mu_i(t_3) + \sqrt{\frac{3\log T}{2N_i'(t_3)}} + \sqrt{\frac{2\log T}{N_i'(t_3)}}\\
    &\le \mu_1 + (\sqrt{2} + \sqrt{3/2})\sqrt{\frac{\log T}{N_i'(t_3)}}\\
    &\le \mu_1 + (\sqrt{2} + \sqrt{3/2})\sqrt{\frac{\log T}{N_i'(t_3) - N_i'(t_1)}}\\
    &\le \mu_1 + \frac{\Delta_k}{2}.
\end{align*}

However, this makes a contradiction with $\hat{\mu}_k(t) + \rad_k(t) \ge \mu_1 + \Delta_k$, since the 2-UCB algorithm always chooses the arm with the largest UCB to pull. 
%
%
%
\end{proof}

\section{Proof of Theorem \ref{Theorem_IR}} \label{appendix H}
\textbf{Theorem \ref{Theorem_IR}.}
    When event $\mathcal{E}$ happens and time horizon $T$ is large enough such that $\frac{T}{\log^2 T} > \frac{N}{4\Delta_{\min}^4}$, applying Eq. \eqref{Eq_5} and Eq. \eqref{Eq_6} as the incentive mechanism in our Balanced-ETC algorithm achieves IR, i.e., for any agent $m$,
    \begin{equation*}
        R_m(T) - \Com_m + \Cost_m - R_{UCB}(T) \le 0. 
    \end{equation*}

\begin{proof}
Firstly, for any agent $m$, once she pulls a sub-optimal arm $i$ (no matter in an explore step or in a commit step), she will suffer from a regret of $\Delta_i$, and receive a compensation of $\sqrt{\frac{8(1+\sqrt{B})^2\log T}{N_i(T)}}$.
Hence, by Lemma \ref{range}, under event $\mathcal{E}$, we could obtain that 
\begin{equation}\label{Eq_10}
    R_m(T) \le \Com_m.
\end{equation}

Secondly, by Lemma \ref{range_ucb}, we know that with high probability, $N_i'(T) \ge {(\sqrt{2}-\sqrt{3/2})^2\log T \over 4\Delta_i^2}$. Hence
\begin{eqnarray*}
    R_{UCB}(T) = \sum_{i=2}^N N_i'(T)\Delta_i \ge \sum_{i=2}^N {(\sqrt{2}-\sqrt{3/2})^2\log T \over 4\Delta_i}.
\end{eqnarray*}
On the other hand, 
\begin{eqnarray*}
    \Cost_m &=& \sum_{i=2}^N N_i(T) \sqrt{\frac{(\sqrt{2}-\sqrt{3/2})^4 \log T}{128(1+\sqrt{B})^2 N_i(T)}}\\
    &\le& 
    \sum_{i=2}^N {(\sqrt{2}-\sqrt{3/2})^2\log T \over 4\Delta_i},
\end{eqnarray*}
where the last inequality holds because of Lemma \ref{range}. 
%
This means that 
\begin{equation}\label{Eq_11}
    \Cost_m \le R_{UCB}(T).
\end{equation}
Along with Eq. \eqref{Eq_10}, we finish the proof of Theorem \ref{Theorem_IR}.
\end{proof}

\section{Experimental details.} \label{appendix I}
The type of compute worker used in our experiment is CPU. Our experimental environment is a machine with 96 Intel(R) Xeon(R) Gold 5220R CPUs @ 2.20GHz, with an x86\_64 architecture.
\subsection{Regret/Time experiments} 
\begin{figure}[h]
    \centering
    \includegraphics[scale = 0.5]{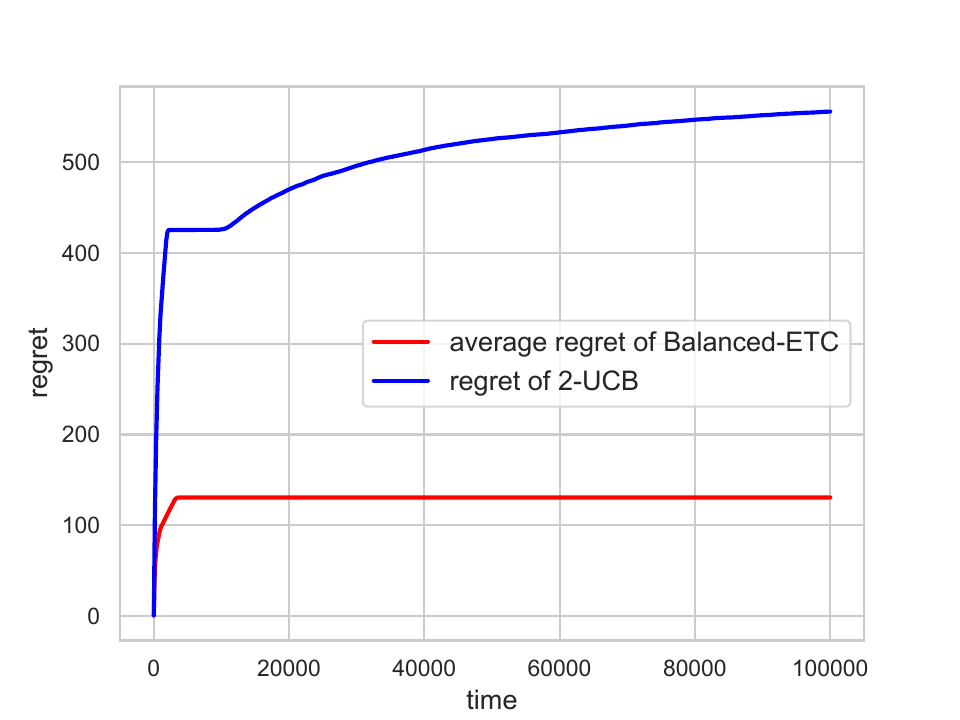}
    \caption{regret with T}
    \label{fig:enter-label}
\end{figure}
To demonstrate that our algorithm can still function properly when one person shares multiple arms and to show the relationship between regret and T, we conduct the following Regret/Time experiments. In this experiments, there are 20 agents and 10 arms with an expected reward vector $\mu = [0.95, 0.85, 0.75, 0.65, 0.55, 0.45, 0.35, 0.25, 0.15, 0.05]$. we set $T = 10^5$, $B = 1$ and each agent shares multiple arms (more than one arm). Specifically, For agent m, her non-sensitive arm set is $\{m \bmod 10, (m + 1) \bmod 10\}$. All these results take an average over 100 independent runs.
We can see that the average regret of Balanced-ETC is significantly less than the regret of 2-UCB, which means our algorithm can effectively reduce the overall regret.

\subsection{Experimental results of random setting}
To enhance the experimental section, we design experiments involving randomized reward distributions and randomized information-sharing structures. These experiments are more practical in nature and aim to validate our theoretical findings. In the setting we considered, the expected reward for each arm is a \textbf{random} value within the interval [0, 1], and the arm that each agent is willing to share is \textbf{randomly} selected. Under the condition that each arm has at least one agent willing to share, we conducted 100 experiments and averaged the results.

\begin{figure}[h]
    \centering
    \includegraphics[scale = 0.5]{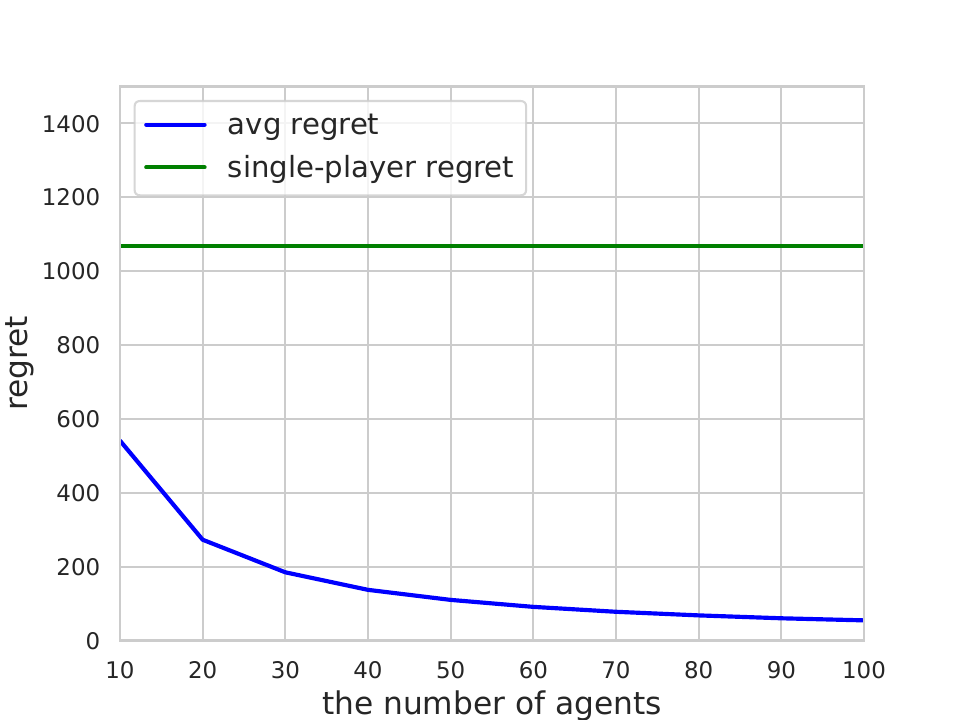}
    \caption{experimental results of random setting}
    \label{fig:enter-label}
\end{figure}

Our experiments confirm our previous findings, showing that the average individual regret (blue
line) keeps decreasing and tends to be zero when there are more
agents involved. This indicates that our algorithm continues to perform effectively in a more practical environment.
\end{appendices}
\end{document}